\pdfoutput=1

\documentclass[11pt]{article}

\usepackage{EMNLP2022}

\usepackage{times}
\usepackage{latexsym}

\usepackage[T1]{fontenc}

\usepackage[utf8]{inputenc}

\usepackage{microtype}

\usepackage{inconsolata}

\usepackage{enumitem}
\usepackage{amsmath, amssymb, amsthm}
\usepackage{listings}
\usepackage{verbatim}
\usepackage{array}
\usepackage{graphicx}
\usepackage{hyperref}

\definecolor{mygreen}{RGB}{63,126,63}
\definecolor{mygray}{RGB}{242,242,242}
\definecolor{myblue}{RGB}{25,25,112}
\definecolor{mymauve}{RGB}{161,113,136}
\definecolor{mypurple}{RGB}{108,48,130}

\lstdefinelanguage{Coq}{
,morekeywords={match,end,Definition,Inductive,Lemma,Theorem,Record,
               Hypothesis,Variable,Section,End,case,of,is,let,in,do,return,with,Extract,Constant,Inlined,Inline,Extraction,Fixpoint,Program,Function,Fix,Class,Local,Output,Input,Proof,Qed}
,morecomment=[s]{(*}{*)}
,keywordstyle=\bfseries\color{myblue}
,rulecolor=\color{black}
,commentstyle={\color{black}}
,columns=fullflexible
,numberstyle=\tiny\color{gray}
,escapeinside={@}{@}
,belowskip=-1\baselineskip
}

\newtheorem*{theorem*}{Theorem}
\newtheorem*{definition*}{Definition}

\title{Towards Autoformalization of Mathematics and Code Correctness: Experiments with Elementary Proofs}

\author{Garett Cunningham \\
  School of EECS\\
  Ohio University\\
  Athens, OH 45701\\
  \texttt{gc974517@ohio.edu} \\
 \\\And
  Razvan C. Bunescu \\
  Department of Computer Science\\
  UNC Charlotte\\
  Charlotte, NC 28223\\
  \texttt{razvan.bunescu@uncc.edu} \\\And
  David Juedes \\
  School of EECS\\
  Ohio University\\
  Athens, OH 45701\\
  \texttt{juedes@ohio.edu} \\}

\begin{document}
\maketitle
\begin{abstract}
The ever-growing complexity of mathematical proofs makes their manual verification by mathematicians very cognitively demanding. Autoformalization seeks to address this by translating proofs written in natural language into a formal representation that is computer-verifiable via interactive theorem provers. In this paper, we introduce a semantic parsing approach, based on the Universal Transformer architecture, that translates elementary mathematical proofs into an equivalent formalization in the language of the Coq interactive theorem prover. The same architecture is also trained to translate simple imperative code decorated with Hoare triples into formally verifiable proofs of correctness in Coq. Experiments on a limited domain of artificial and human-written proofs show that the models generalize well to intermediate lengths not seen during training and variations in natural language.
\end{abstract}

\section{Introduction}

To the uninitiated, the notion of mathematical proof represents simply an argument written by people to convince others of mathematical truth.    However, in a real sense, mathematical proof must have formal underpinnings that go beyond the written argument.    Arguments that lack such underpinnings might have fatal errors or even logical inconsistencies (see, for example, Russell's Paradox \citep{sep-russell-paradox}).   Nevertheless, mathematical arguments written in natural language are the norm and they have great value.

In \citet{Tymoczko1979-TYMTFP}'s well-known paper that discusses a somewhat controversial (at the time) proof of the Four Color Theorem \citep{AppelHaken1,AppelHaken2}, he explores ``what is a mathematical proof?''  He posits that all mathematical proofs must be (i) convincing, (ii) surveyable, and (iii) formalizable.   The first two points are for the reader---proofs must be convincing to and comprehensible by mathematicians.   For the third point, he notes that, ``Most mathematicians and philosophers believe that any acceptable proof can be formalized. We can always find an appropriate formal language and theory in which the informal proof can be embedded and `filled out' into a rigorous formal proof.''   For most mathematicians, this third part is crucial for ensuring that subtle, but fatal, errors in logic do not exist in mathematical proof.

Great progress has been made since the 1970's in fully formalizing significant mathematical results.   For instance, the Feit-Thompson Theorem \citep{Gonthier_Odd_Order, Gonthier20131} and the Four Color Theorem \citep{Gonthier08formalproof} have been formally verified using the proof assistant Coq \cite{bertot2013interactive}, and the Kepler Conjecture \citep{Hales2005,Hales2017} has been formally verified using the proof assistants Isabelle and HOL Light \citep{nipkow2002isabelle}. Moreover, proof assistants have demonstrated immense utility for software verification, such as the full certification of a C compiler \citep{leroy2009formal}. Proofs demonstrating the correct behavior of code share a similar structure to proofs in pure mathematics, where systems like Hoare logic replace standard first-order logic. Thus, Tymoczko's criteria for mathematical proof can be extended to the verification of programs. For many experts, LaTeX  provides an excellent tool for satisfying the first two criteria.   In addition, carefully written LaTeX  \citep{higham2020} provides a rich structure for establishing the third criterion.

The vast majority of modern mathematics is expressed using natural language (NL), with the overwhelming majority typeset in LaTeX.   Fully formalizing mathematics using proof assistants is still a difficult and time consuming task.    This paper takes some preliminary steps toward bridging this gap
by exploring how modern machine learning techniques 
can be used to convert carefully written LaTeX  into equivalent, and formally verified mathematics in Coq, a process referred to as \emph{autoformalization} in the literature \citep{Szegedy2020}.  

\citet{latex_mizar2018,latex_mizar} explored the similar task of translating mathematical statements from LaTeX into Mizar, using LSTM-based models with attention. To generate aligned LaTeX-Mizar pairs, they use a tool \cite{bancerek:mm06} that translates top-level Mizar statements into artificial LaTeX sentences, a task that is facilitated by the fact that Mizar is human readable and similar in length with the corresponding LaTeX version. \citet{carman2021translating} evaluated the competency of LSTMs toward formalizing a restricted set of artificially generated theorems about simple arithmetic expressions, reporting reasonable success over expression lengths seen during training. More recently, \citet{wu2022autoformalization} evaluated Codex and PaLM on a significantly more limited, but human-written set of theorems in algebra and number theory.


In contrast to prior work, we address the autoformalization of both {\it theorems and their proofs}, and extend the scope to {\it proofs of code correctness}. We use a number of manually written mathematical statements to abstract a complex grammar that is then used to generate a dataset of substantially longer and more diverse mathematical theorems and proofs. We develop an architecture based on the Universal Transformer \citep{dehghani2018universal} and adapt a copying mechanism \citep{gu2016incorporating} to handle arbitrary numbers and variable names at test time. The models are evaluated extensively on their ability to systematically generalize to statement lengths not seen during training, for which we report {\it sequence-level} accuracy as well as a \emph{semantic-level} accuracy calculated by combining sequence-level accuracy for the theorem and running Coq to determine if the generated proof is correct. Code and data are made available at \href{https://github.com/gc974517/autoformalization}{https://github.com/gc974517/autoformalization}.



\section{Dataset of Theorems and Proofs}
\label{sec:data}

We create two independent datasets of mathematical statements that overall correspond to four classes of theorems and proofs: the first dataset contains three classes of arithmetic statements ({\sc even-odd}, {\sc composites}, and {\sc powers}), described in detail in Section~\ref{sec:dataset-math}, and the second dataset containing statements about code correctness via Hoare logic ({\sc poly}), described in detail in Section~\ref{sec:dataset-code}. In each example, the input theorem-proof pair is given in LaTeX, whereas the formalized output is represented in Coq. This work focuses on the proof assistant Coq \cite{bertot2013interactive} because (a) there is a rich set of mathematical libraries that have been developed for it, (b) it has been used successfully to reason about significant computation artifacts, such as the ComperCert C compiler \cite{leroy2009formal}), and (c) it benefits from a rich set of training material for the proof assistant related to software verification \cite{pierce2010software}.

Each class of examples demonstrates features necessary for the successful autoformalization of mathematical theorems and proofs. For example, \textsc{powers} and \textsc{composites} examples may define useful terminology to make the theorems shorter, e.g. proving that 4 is a {\it square}, or conversely they may state theorems directly without any preliminary definitions, e.g. proving $\exists n. \, n^2 = 4$. As shown in Figures~\ref{fig:composites-example} and~\ref{fig:powers-example},  this corresponds in Coq to aliasing propositions using the \texttt{Definition} keyword. Additionally, the examples in the dataset provide a stress test of the copying mechanism described in Section~\ref{sec:copying}, testing its ability to learn the correct order and number of terms to include in mathematical expressions, as well as their placement in theorems and proofs, in a way that generalizes to arbitrary tokens in mathematical language.

For each of the four classes of theorems and proofs, we manually created a few examples ourselves in order to guide the construction of a complex grammar that is then used to generate a dataset of substantially longer and more diverse mathematical theorems and proofs. Each dataset is generated using its corresponding grammar in an identical way. First, a random seed is sampled that controls the overall structure of the theorem, proof, and definition, if any. Then, the skeleton structure of the proof is completed with phrases that are sampled from a separate context-free grammar. The coarse control of the skeleton structure allows the construction of examples with interesting features like sublemmas, forward or backward proof direction, coreference, or additional conditions for the theorem, among others. 

\begin{figure}[t!]
    \centering
    \small
    \begin{tabular}{p{0.45\textwidth}}
        \multicolumn{1}{c}{LaTeX Input Sequence} \\
        \cline{1-1}
        \begin{theorem*}
            $28 M + 308$ is even.
        \end{theorem*}
        \begin{proof}
            We know the summation between even numbers in $\mathbb{N}$ will be an even number. Observe that $308$ is known to be even. Additionally, note that the pair $M \times 28$ is trivially even. This is true because the coefficient $28$ is even.
        \end{proof} \\
        \multicolumn{1}{c}{Coq Output Sequence} \\
        \cline{1-1}
        \begin{lstlisting}
Require Import Arith.
Theorem M28_308:
  forall M : nat, Nat.even(28 * M + 308) = true.
Proof.
  intros.
  repeat rewrite Nat.even_add.
  assert (H1: Nat.even 308 = true).
  { auto. }
  assert (H2: Nat.even 28 = true).
  { auto. }
  assert (H3: Nat.even (28 * M) = true).
  { rewrite Nat.even_mul.
    rewrite H2.
    auto. }
  rewrite H1.
  rewrite H3.
  auto.
Qed.
        \end{lstlisting}
    \end{tabular}
    \vspace*{-3em}
    \caption{Generated example from the \textsc{even-odd} set.}
    \label{fig:ltc-example}
\end{figure}

\begin{figure}[t!]
\small
    \begin{tabular}{p{0.45\textwidth}}
        \multicolumn{1}{c}{LaTeX Input Sequence} \\
        \cline{1-1}
        \begin{theorem*}
            $450 + a \cdot 192 + j \cdot 462$ is guaranteed to be even for any natural terms $j$, and $a$.
        \end{theorem*}
        \begin{proof}
            It can be justified that $192 \cdot a + j \cdot 462$ is trivially even. Note that $192 a$ is an even number in $\mathbb{N}$ because multiplying between an even integer with an arbitrary number in $\mathbb{N}$ is guaranteed to be even. Likewise, $462 j$ is trivially an even number in $\mathbb{N}$. The claim is proven as a consequence of the fact that the sum of even numbers with an even number will be in itself an even number. Therefore, our theorem holds.
        \end{proof} \\
        \multicolumn{1}{c}{Coq Output Sequence} \\
        \cline{1-1}
        \begin{lstlisting}
Require Import Arith.
Theorem a450_192j_450_even:
  forall j a : nat,
  Nat.even (192 * a + 462 * j + 450) = true.
Proof.
  intros.
  rewrite Nat.even_add.
  assert (H1: Nat.even (192 * a) = true).
  { rewrite Nat.even_mul.
    auto. }
  assert (H2: Nat.even (462 * j) = true).
  { rewrite Nat.even_mul.
    auto. }
  assert (H3: Nat.even
    (192 * a + 462 * j) = true).
  { repeat rewrite Nat.even_add.
    rewrite H1.
    rewrite H2.
    auto. }
  rewrite H3.
  auto.
Qed. 
        \end{lstlisting}
    \end{tabular}
    \vspace*{-3em}
    \caption{Instance of sublemma use in the \textsc{even-odd} dataset. The proof that the sum of non-constant terms is even (assertion \lstinline|H3|) is given before proving the theorem.}
    \label{fig:coref-example}
\end{figure}

Many of the difficulties in formalizing mathematical statements from NL into Coq stem from the wide variability in the level of detail of mathematical proofs, and the frequent mismatch between what is considered an acceptable inference step in NL proofs vs. an inference step in Coq. 
Furthermore, there may be multiple Coq proofs for any given theorem, at different levels of granularity. We address this ambiguity by requiring the structure of the Coq proof to match the overall structure of the NL proof. This is achieved by quasi-synchronously generating the LaTeX and Coq versions of mathematical statements, while still allowing for some simple re-orderings in order to improve generalization performance, e.g. swapping arguments of commutative operations.

In total, the grammar-based method for generating examples can theoretically produce over 283 million unique arithmetic examples and over 491,000 unique code examples, before considering variations in phrasing by sampling from the context-free grammar.

\subsection{Arithmetic Statements}
\label{sec:dataset-math}

We generated three classes of mathematical statements, i.e. theorem-proof pairs:
\begin{itemize}[topsep=4pt,itemsep=2pt,partopsep=2pt, parsep=2pt]
    \item \textsc{even-odd}: an expression is even or odd.
    \item \textsc{composites}: a number is composite.
    \item \textsc{powers}: a number is an integer power of $n$.
\end{itemize}
\textsc{even-odd} examples contain arithmetic expressions of $n$ variables with even coefficients that are summed with a constant term, meaning that the parity of this constant determines the parity of the whole expression. Proofs make use of this fact with varying rigor based on our manually designed grammar, an example of which is shown by Figure \ref{fig:ltc-example}. The Coq program is generated concurrently with the paired LaTeX example. The example shown in Figure~\ref{fig:coref-example} illustrates the generation and use of prior facts to prove an implicit sublemma, in both the natural language and matching Coq version.

Examples of theorems and proofs for \textsc{powers} and \textsc{composites} share a similar structure in both their LaTeX and Coq forms, as shown in Figures~\ref{fig:composites-example} and \ref{fig:powers-example}, respectively. The theorems assert the existence of a natural number such that a defining property holds and their proofs are constructive, with the distinction that examples for composites prove factorization into $n$ factors. 

\begin{figure}[t!]
    \centering
    \small
    \begin{tabular}{p{0.45\textwidth}}
        \multicolumn{1}{c}{LaTeX Input Sequence} \\
        \cline{1-1}
        \begin{definition*}
            We define that $w \in \mathbb{N}$ is a \emph{composite} natural number if taking some $R$, $Q \in \mathbb{N}$ we have $Q$, $R \geq 2$ and $Q \times R = w$.
        \end{definition*}
        \begin{theorem*}
            $35$ is a \emph{composite} whole number.
        \end{theorem*}
        \begin{proof}
            Remember that a composite natural number is the multiplication between $Q$ and $R$ such that $Q$ and $R \geq 2$. Allow $R = 7$, $Q = 5$. We justify the result is valid as $35 = Q \cdot R$.
        \end{proof} \\
        \multicolumn{1}{c}{Coq Output Sequence} \\
        \cline{1-1}
        \begin{lstlisting}
Require Import Lia.
Definition composite
  (w : nat) :=
    exists R Q : nat,
    (Q >= 2) /\ (R >= 2)
    /\ (Q * R = w).
Theorem w_composite:
  composite 35.
Proof.
  unfold composite.
  exists 7.
  exists 5.
  lia.
Qed.
        \end{lstlisting}
    \end{tabular}
    \vspace*{-3em}
    \caption{Generated \textsc{composites} example.}
    \label{fig:composites-example}
\end{figure}

\begin{figure}[t!]
    \centering
    \small
    \begin{tabular}{p{0.45\textwidth}}
        \multicolumn{1}{c}{LaTeX Input Sequence} \\
        \cline{1-1}
        \begin{definition*}
            We say $o$ is a \emph{square} whenever there exists some whole number $Z$ such that $Z \geq 2$ and $o = Z^2$.
        \end{definition*}
        \begin{theorem*}
            $o = 64$ is a \emph{square}.
        \end{theorem*}
        \begin{proof}
            Let $Z = 8$. Observe that $64 = 8^2$. Also notice $Z = 8$ is more than or equal to 2. This yields 64 is a square whole number.
        \end{proof} \\
        \multicolumn{1}{c}{Coq Output Sequence} \\
        \cline{1-1}
        \begin{lstlisting}
Require Import Lia.
Definition square
  (o : nat) :=
    exists Z : nat,
    (Z >= 2) /\ (o = Z^2).
Theorem square_64:
  square 64.
Proof.
  unfold square.
  exists 8.
  assert (H1: 8 >= 2).
  { lia. }
  repeat split.
  apply H1.
Qed.
        \end{lstlisting}
    \end{tabular}
    \vspace*{-3em}
    \caption{Generated example from the \textsc{powers} set.}
    \label{fig:powers-example}
\end{figure}

For both training and testing, we generate 5,000 even-odd, 5,000 composites, and 2,000 powers examples. We train on values of $n \in \{2, 3, 5, 7, 9\}$ and test on values $n \in \{2, 3, \ldots, 12\}$, where $n$ represents the number of variables in the arithmetic expression, the number of factors, or the power, respectively. This is done in order to evaluate the model's ability to generalize to unseen arithmetic expression lengths and numbers of factors.

\begin{figure*}[t!]
    \centering
    \begin{tabular}{p{0.45\textwidth}cp{0.45\textwidth}}
        \multicolumn{1}{c}{LaTeX Input Sequence} && \multicolumn{1}{c}{Coq Output Sequence} \\
        \cline{1-1}\cline{3-3}
        \begin{theorem*}
            Consider the following series of commands such that
            \begin{lstlisting}
S := 3;
S := 3 + S * Z;
S := 1 + S * Z\end{lstlisting}
            Allow $Z = y$, for any natural number $y$, ahead of running this code then $S = 3 \times y^2 + 3 \times y + 1$ after the set of instructions has executed.
        \end{theorem*}
        \begin{proof}
            By application of usual Hoare logic:
            \begin{eqnarray*}
                \{Z = y\} \\
                \text{\lstinline|S := 3;|} \\
                \{Z = y \land S = 3\} \\
                \text{\lstinline|S := 3 + S * Z;|} \\
                \{Z = y \land S = 3 \times y + 3\} \\
                \text{\lstinline|S := 1 + S * Z|} \\
                \{Z = y \land S = 3 \times y^2 + 3 \times y + 1\}
            \end{eqnarray*}
            Hence, this program is shown to be correct.
        \end{proof} &&
        \begin{lstlisting}
Require Import String.
From PLF Require Import Imp.
From PLF Require Import Hoare.
Theorem poly_code_correct:
  forall y : nat,
  {{ Z = y }}
  S := 3;
  S := 3 + S * Z;
  S := 1 + S * Z
  {{ S = 3 * y ^ 2 + 3 * y + 1 }}.
Proof.
  intros.
  apply hoare_seq with
    (Q := (
      (Z = y /\ S = 3)
    )%assertion).
  apply hoare_seq with
    (Q := (
      (Z = y /\ S = 3 * y + 3)
    )%assertion).
  apply hoare_seq with
    (Q := (
      (Z = y /\ S = 3 * y^2 + 3 * y + 1)
    )%assertion).
  all: eapply hoare_consequence_pre;
  try (apply hoare_asgn || assn_auto'').
Qed.
        \end{lstlisting}
    \end{tabular}
    \vspace*{-2em}
    \caption{Generated \textsc{poly} example: [Left] the Hoare logic proof; [Right] the code correctness proof in Coq.}
    \label{fig:poly-example}
\end{figure*}

\subsubsection{Handwritten Examples}

We also created a small collection of 45 human-written LaTeX theorem-proof pairs to evaluate performance on examples outside of our manually generated grammar. These are distinct from the original manually written examples that were used to guide the development of the generative grammar. There are 15 examples for each type of proof from the arithmetic set, using the same vocabulary with a number of unseen grammatical structures.

\subsection{Code Correctness Statements}
\label{sec:dataset-code}

We create a dataset of correctness proofs about short programs written in the imperative programming language \emph{Imp} \citep{Pierce:SF2}, which we call \textsc{poly}. The programs represent various algorithms for evaluating a polynomial, and their proofs of correctness verify that the programs correctly model the polynomial as a mathematical function. Proofs are conducted as either fully decorated programs or as sequences of Hoare triples with natural language justifying steps in between. An example is shown in Figure~\ref{fig:poly-example}.

For both training and testing data, we generate 5,000 examples. We train on programs containing 2, 3, 5, 7, 9, and 11 lines, then test on programs containing from 2 up to 14 lines to evaluate the model's ability to generalize to novel program lengths.

\let\vec\mathbf
\section{Semantic Parsing Architecture}
\label{sec:model}

To formalize LaTeX statements into Coq,  we developed an encoder-decoder architecture based on the Universal Transformer \cite{dehghani2018universal}. Similar to \citet{csordas2021devil}, we do so by adding recursive passes into the encoder and decoder of a base Transformer \citep{vaswani2017attention}, thus making the model analogous to a Universal Transformer without adaptive computation time (ACT). Further, we introduce a copying mechanism and support for out-of-vocabulary mathematical terms.

\subsection{Copying Mechanism}
\label{sec:copying}

Mathematical language contains features uncommon or non-existent in natural language, such as numbers, variables, and carefully defined terminology. In order to address the use of general mathematical jargon, these tokens are replaced in the LaTeX input with generic forms denoting their usage, such as \texttt{<var1>} up to \texttt{<varN>} for variables, which effectively ensures {\it generalization to variable renaming} \cite{ferreira_be_2022}, \texttt{<nat1>} up to \texttt{<natN>} for numbers, or \texttt{<def>} for definitions, coupled with the use of a copying mechanism adapted from \citet{gu2016incorporating}. Note that a different generic token is introduced for each unique numerical constant or variable literal in the theorem and its proof, and the corresponding generic token is used in the Coq version. For example, considering the $\langle$LaTeX, Coq$\rangle$ pair in Figure~\ref{fig:composites-example}, \texttt{<nat1>}, \texttt{<nat2>}, \texttt{<nat3>}, and \texttt{<nat4>} would be used to replace the constants 2, 35, 7, and 5 respectively, everywhere in the LaTeX and Coq statements. Similarly, \texttt{<var1>}, \texttt{<var2>}, and \texttt{<var3>} were used to replace variable literals $w$, $R$, and $Q$. This is in contrast to using just two generic tokens \texttt{<nat>} and  \texttt{<var>} everywhere, which would make all numbers coreferent and all variables coreferent. Preliminary experiments validated the utility of encoding these distinctions while maintaining the correct coreference in both LaTeX and Coq statements. 

Overall, by using generic tokens for numbers, variables, and definitions, only a limited set of embeddings need to be trained and the model is forced to utilize contextual information in order to appropriately copy tokens into the Coq output. In this way, the model has the ability to generalize to unseen numbers or variable and definition names.

The original CopyNet \citep{gu2016incorporating} used an encoder-decoder architecture with a copying mechanism to calculate the probabilities of generating in-vocabulary tokens vs. copying tokens from the input sequence to the output. Our autoformalization task guarantees mutual exclusivity between generating (g) and copying (c) tokens, which allows using a simplified formula for calculating the probability of producing a token $y_t$ at time step $t$. Letting $\mathcal{V}_c$ denote the Coq vocabulary, $X$ denote the input sequence of LaTeX tokens, and $\mathcal{X}$ denote the collection of unique tokens in $X$, we calculate the probability of producing $y_t$ as:
\begin{align}
    p(y_t) &=
    \begin{cases}
        p(y_t, \text{g}) = \displaystyle\frac{1}{Z_t} e^{\psi_g(y_t)}, \hfill y_t \in \mathcal{V}_c \\
        p(y_t, \text{c}) = \displaystyle\frac{1}{Z_t} \displaystyle\sum_{x_j \in X: x_j = y_t} \!\!\! e^{\psi_c(x_j)}, \hfill y_t \in \mathcal{X}
    \end{cases} \nonumber
\end{align}
where $Z_t = \displaystyle\sum_{y_t \in \mathcal{V}_c} e^{\psi_g(y_t)} + \sum_{x_j \in X} e^{\psi_c(x_j)}$. The scoring functions 
are given by $\psi_g(y_t) = \vec{v}_{y_t}^\top \vec{W}_o \vec{s}_t$ and $\psi_c(x_j) = \tanh{\left(\vec{h}_j^\top \vec{W}_c\right)} \vec{s}_t$,
where $\vec{v}_{y_t}$ is a one-hot encoding of $y_t$, $\vec{h}_j$ is the hidden encoder state for the input token $x_j$, $\vec{s}_t$ is the decoder state at step $t$, and $\vec{W}_o$ and $\vec{W}_c$ are learnable parameters.


\subsection{Encoder-Decoder Architecture}

We diverge from the standard Transformer architecture in a few crucial ways:
\begin{itemize}[topsep=4pt,itemsep=2pt,partopsep=2pt, parsep=2pt]
    \item Probabilities are calculated via $p(y_t)$ above.
    \item Absolute positional encodings are removed.
    \item Self-attention uses relative positional representations as in \citet{shaw-etal-2018-self}.
    \item Stacks of $N$ encoder/decoder blocks have $T$ recurrent passes.
\end{itemize}
All other aspects of the model remain unchanged from the original Transformer. We emphasize relative positional information over absolute in our model architecture. Preliminary evaluations on the \textsc{even-odd} dataset showed that Transformer models that use absolute positional encodings obtain 0\% sequence-level accuracy on expression lengths that are not seen at training time. Removing reliance on absolute position resolves this type of systematic generalization. The use of relative positional encodings for the Transformer-based models was thus essential for achieving stronger systematic generalization, which also agrees with the findings of \citet{csordas2021devil} on other NLP tasks.


\section{Experimental Evaluations}
\label{sec:evaluation}

To evaluate the performance of trained models, we ran two primary experiments: first on the collection of arithmetic examples, then on the collection of code correctness examples.
All models are evaluated in terms of {\it sequence-level} accuracy, where an example is considered correctly processed only if the generated Coq sequence for both the theorem and its proof perfectly matches token by token the ground truth sequence. We also report {\it semantic-level} accuracy, for which the generated Coq theorem needs to attains a perfect sequence-level accuracy and the Coq engine verifies that the generated Coq proof truly proves the generated Coq theorm, regardless of whether it matches the ground truth version of the proof. This emphasizes that the model was able to capture the general meaning of the natural language proof by correctly translating the theorem and successfully proving it using the natural language version as a guide.

All experiments were performed on one NVIDIA RTX-A6000 GPU with 48GB of memory.

\subsection{Arithmetic Statements}
\label{sec:eval-math}

We evaluate a Transformer model on the full data combining \textsc{even-odd + composites + powers} and using both the theorem and its proof in each sequence. We tune a model with embedding and state sizes of 32, a feed forward width of 256, 4 encoder and decoder blocks with 4 recurrent passes, 4 attention heads, and a clipping value of 2 for self-attention. We trained this model over minibatches of size 20, optimized with Adam using $\beta_1 = 0.9$, $\beta_2 = 0.98$, $\varepsilon = 1e-9$, and an initial learning rate of 0.001, annealed by a factor of $1/\sqrt{10}$ based on training loss plateaus with a patience of 5 epochs. 

The results in Table \ref{tab:results} show that the model generalizes well to the intermediate lengths of $\{4, 6, 8\}$, with a small number of correctly translated examples longer than the maximum of 9 used in training. Otherwise, the model fails to generalize to longer unseen lengths, which is not surprising, given that Transformer models are known to fail dramatically at systematic generalization on longer inputs for various NLP tasks \cite{csordas2021devil}, or to incur substantial decrease in accuracy for longer symbolic integration problems \cite{welleck_symbolic_2022}. Switching to semantic-level evaluation leads to a significant increase in accuracy for \textsc{composites}, with a more modest increase for \textsc{even-odd}.

\begin{table}[t]
    \centering
    \begin{tabular}{c|rr|rr|r|}
        \multicolumn{1}{c}{} & \multicolumn{2}{c}{\textsc{even-odd}} & \multicolumn{2}{c}{\textsc{composites}} & \multicolumn{1}{c}{\textsc{poly}} \\
        \cline{2-6}
        \multicolumn{1}{c|}{$n$} & Seq & Sem & Seq & Sem & Both \\
        \hline
         2  &  99.6 &  99.8 &  76.7 &  97.6 & 100.0 \\
         3  &  99.4 &  99.6 &  64.6 &  94.2 & 100.0 \\
         4  &  99.4 &  99.4 &  56.1 &  93.9 & 82.1 \\
         5  &  99.2 &  99.6 &  54.9 &  94.4 & 99.2 \\
         6  &  98.8 &  98.8 &  57.1 &  94.3 & 45.1 \\
         7  &  99.1 &  99.5 &  58.5 &  93.4 & 96.5 \\
         8  &  93.8 &  94.0 &  53.5 &  88.3 & 15.7 \\
         9  &  98.6 &  98.6 &  53.7 &  93.7 & 98.2 \\
        10  &   7.0 &   7.0 &   1.2 &   1.6 & 35.6 \\
        11  &   0.0 &   0.0 &   0.0 &   0.0 & 93.5 \\
        12+ &   0.0 &   0.0 &   0.0 &   0.0 &  0.0 \\
        \cline{2-6}
        \multicolumn{2}{c}{} \\
        \cline{3-6}
        \multicolumn{2}{r|}{\textsc{powers}} & \multicolumn{2}{c}{Seq = 100.0} & \multicolumn{2}{c|}{Sem = 100} \\
        \cline{3-6}
    \end{tabular}
    \caption{Sequence-level (Seq) and semantic-level (Sem) accuracy (\%) on test examples, split by expression length, with the exception of \textsc{powers}.}
    \label{tab:results}
\end{table}

\subsection{Code Correctness Statements}

We extend our scope to include data representing proofs of program correctness using the language of Hoare logic. We train a separate model with the same embedding and state sizes, feed forward width, and learning rates as in Section~\ref{sec:eval-math}. Depth is increased to 8 encoder and decoder blocks with 8 recurrent passes, 8 attention heads, and a clipping value of 8. The model is trained over minibatches of size 1 with Adam, with a patience of 3 epochs.

The \textsc{poly} results shown in Table~\ref{tab:results} demonstrate that the model is able to generalize to program line counts of $\{4, 6, 8, 10\}$ unseen during training with diminishing returns as the program length grows, eventually failing to generalize for lengths longer than the maximum seen in training. We observe that increasing the depth of the model significantly improved generalization. A model with identical hyperparameters to the arithmetic experiment yielded less then half the sequence-level accuracy for intermediate program lengths. Therefore, further increasing the depth of the model could push performance closer to optimal generalization to intermediate lengths at the cost of significantly more computing resources. Additionally, \textsc{poly} examples are far less prone to non-fatal token swapping errors. We observe that semantic-level accuracy is identical to sequence-level, as all copying errors compromised the validity of the proof. Therefore, accuracies are shown as one column (Both).

\subsection{Handwritten Examples}

We also evaluate the semantic-level accuracy of the trained models on the collection of 45 human-written LaTeX theorem-proof pairs. This is done by manually verifying that the generated Coq theorem corresponds to the LaTeX version and that the subsequent proof is correct according to the Coq interpreter. The fully trained model achieved 53.3\% for both \textsc{even-odd} and \textsc{composites}, and 73.3\% for \textsc{powers}.

Mistakes in almost all cases are confined to the mishandling of out-of-vocabulary tokens, such as mis-copying a variable within a definition or the omission of an assertion in the proof tied to a term. The model otherwise generated syntactically sound Coq code. Mistakes strongly correlate with examples that deviate significantly from the grammatical structure of the artificial data. Thus, pre-trained language models as evaluated by \citet{wu2022autoformalization} or pre-training new models on mathematical corpora like MATH \citep{hendrycksmath2021} may serve to alleviate the problems caused by the scarcity of aligned natural and formal mathematics data.


\section{Concluding Remarks}

As we have seen, it is feasible to train machine learning models to perform autoformalization over very restricted domains of math and code correctness proofs. These models show capability to systematically generalize to new expression lengths and program sizes. Moreover, these models were able to translate previously unseen hand written natural language examples, albeit with lower accuracy.  We are hopeful that this approach can be applied to autoformalization of a larger segment of mathematics and code verification.

As mentioned by \citet{Szegedy2020}, "Autoformalization is not just a challenge: successful autoformalization would represent a breakthrough for general AI with significant implications in various domains." We see an especially significant impact in education, where  integration of autoformalization into proof assistants for introductory mathematics and software verification courses would enable the detection of missing steps or misconceptions in students' proofs.


\bibliography{refs}
\bibliographystyle{acl_natbib}

\appendix

\end{document}